\documentclass{article}

\usepackage{times}
\usepackage{graphicx} \usepackage{subfigure} 

\usepackage{natbib}

\usepackage{algorithm}
\usepackage{algorithmic}

\usepackage{amssymb,amsmath,dsfont,multirow}
\usepackage{amsthm}
\usepackage{thmtools}
\usepackage{tikz}

\usepackage{hyperref}

\usepackage[accepted]{icml2017}

\icmltitlerunning{Neural networks and rational functions}

\numberwithin{equation}{section}
\declaretheorem[numberlike=equation]{theorem}
\declaretheorem[numberlike=theorem]{lemma}
\declaretheorem[numberlike=theorem]{proposition}
\declaretheorem[numberlike=theorem]{corollary}
\declaretheoremstyle[qed={\ensuremath\Diamond}]{remstyle}

\usepackage{cleveref}
\usepackage{commath}
\def\R{\mathbb R}
\def\1{\mathds 1}
\def\cO{\mathcal O}
\def\eps{\epsilon}
\def\srelu{\sigma_{\textup{r}}}

\def\poly{\textup{poly}}
\def\polylog{\textup{poly\,log}}

\begin{document} 

\twocolumn[
\icmltitle{Neural networks and rational functions}

\icmlsetsymbol{equal}{*}

\begin{icmlauthorlist}
  \icmlauthor{Matus Telgarsky}{uiuc}
\end{icmlauthorlist}

\icmlaffiliation{uiuc}{University of Illinois, Urbana-Champaign; work completed while visiting the Simons Institute}

\icmlcorrespondingauthor{your friend}{mjt@illinois.edu}

\icmlkeywords{boring formatting information, machine learning, ICML}

\vskip 0.3in
]

\printAffiliationsAndNotice{}  
\begin{abstract}
  Neural networks
  and rational functions efficiently approximate each other.
  In more detail, it is shown here
  that for any ReLU network, there exists a rational
  function of degree $\cO(\polylog(1/\eps))$ which is $\eps$-close,
  and similarly for any rational function there exists a ReLU
  network of size $\cO(\polylog(1/\eps))$ which is $\eps$-close.
  By contrast, polynomials need degree $\Omega(\poly(1/\eps))$ to
  approximate even a single ReLU.
  When converting a ReLU network to a rational function as above,
  the hidden constants depend exponentially on the number of layers,
  which is shown to be tight; in other words, a compositional representation can be beneficial even for rational functions.
\end{abstract}

\section{Overview}

Significant effort has been invested in characterizing the functions
that can be efficiently approximated by neural networks.
The goal of the present work is to characterize neural networks more
finely by finding a class of functions which is not only well-approximated
by neural networks, but also well-approximates neural networks.

The function class investigated here is the class
of \emph{rational functions}: functions represented as the ratio of
two polynomials, where the denominator is a strictly positive
polynomial.  For simplicity, the neural networks are taken to
always use ReLU activation $\srelu(x) := \max\{0,x\}$; for a review
of neural networks and their terminology, the reader is directed
to \Cref{sec:notation}.  For the sake of brevity, a network with
ReLU activations is simply called a \emph{ReLU network}.

\subsection{Main results}

The main theorem here states that ReLU networks and rational functions
approximate each other well in the sense that $\eps$-approximating
  one class with the other requires a representation whose size is
  polynomial in $\ln(1\,/\,\eps)$, rather than being polynomial in $1/\eps$.

\begin{figure}[t]
  \centering
  \includegraphics[width=\columnwidth]{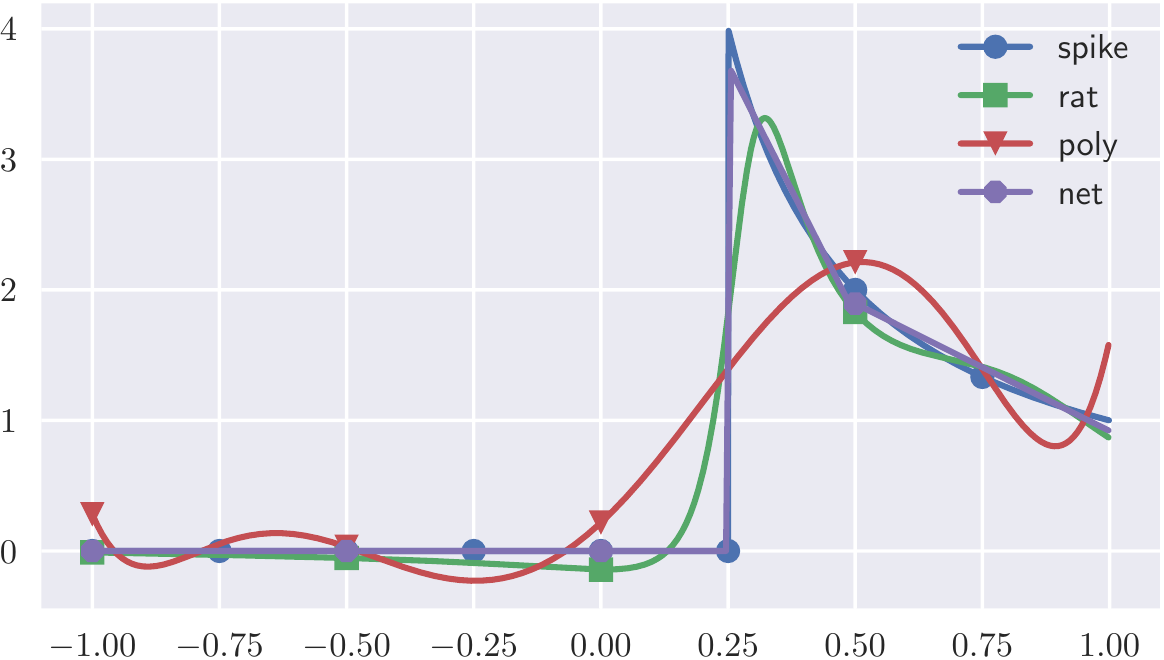}   \caption{Rational, polynomial, and ReLU network fit
    to ``spike'', a function which is $1/x$ along $[1/4,1]$ and $0$ elsewhere.}
  \label{fig:recip}
\end{figure}

\begin{theorem}
  \label{fact:main}
  \begin{enumerate}
    \item
      Let $\epsilon \in (0,1]$ and nonnegative integer $k$ be given.
      Let $p :[0,1]^d \to [-1,+1]$ and $q : [0,1]^d \to [2^{-k}, 1]$
      be polynomials of degree $\leq r$, each with $\leq s$ monomials.
      Then there exists a function $f : [0,1]^d \to \R$,
      representable as a ReLU network
                  of size (number of nodes)
      \begin{align*}
        &\cO\Big( k^7 \ln(1\,/\,\eps)^3
        \\
        &\qquad+ \min\cbr{ srk\ln(sr\,/\,\eps), sdk^2 \ln(dsr\,/\,\eps)^2 }
        \Big),
      \end{align*}
      such that
      \[
        \sup_{x\in[0,1]^d} \envert{ f(x) - \frac {p(x)}{q(x)} }
        \leq \eps.
      \]
            \item
      Let $\epsilon \in (0,1]$ be given.
      Consider a ReLU network $f: [-1,+1]^d \to \R$ with at most $m$
      nodes in each of at most $k$ layers,
      where each node computes $z\mapsto \srelu(a^\top z + b)$
      where the pair $(a,b)$ (possibly distinct across nodes) satisfies
      $\|a\|_1 + |b| \leq 1$.
      Then there exists a rational function $g : [-1,+1]^d\to \R$
      with degree (maximum degree of numerator and denominator)
      \[
        \cO \del{ \ln(k/\eps)^k m^k }
      \]
      such that
      \[
        \sup_{x\in[-1,+1]^d} \envert{ f(x) - g(x) } \leq \eps.
      \]
  \end{enumerate}
\end{theorem}

Perhaps the main wrinkle is the appearance of
$m^k$ when approximating neural networks by rational functions.
The following theorem shows that this dependence is tight.

\begin{theorem}
  \label{fact:descartes}
  Let any integer $k \geq 3$ be given.
  There exists a function $f:\R\to\R$ computed by a
  ReLU network with $2k$ layers, each with $\leq 2$ nodes,
  such that any rational function $g:\R\to\R$ with $\leq 2^{k-2}$
  total terms in the numerator and denominator must satisfy
  \[
    \int_{[0,1]} |f(x) - g(x)|\dif x \geq \frac 1 {64}.
  \]
\end{theorem}

Note that this statement implies the desired difficulty of approximation,
since a gap in the above integral ($L_1$) distance implies a gap
in the earlier uniform distance ($L_\infty$), and furthermore
an $r$-degree rational function necessarily has $\leq 2r+2$ total
terms in its numerator and denominator.

As a final piece of the story, note that the conversion between
rational functions and ReLU networks is more seamless if instead
one converts to \emph{rational networks}, meaning neural networks
where each activation function is a rational function.

\begin{lemma}
  \label{fact:rational_net}
  Let a ReLU network $f :[-1,+1]^d\to\R$ be given as in \Cref{fact:main},
  meaning $f$ has at most $l$ layers
  and each node computes $z\mapsto \srelu(a^\top z + b)$ where
    where the pair $(a,b)$ (possibly distinct across nodes) satisfies
  $\|a\|_1+|b| \leq 1$.
  Then there exists a rational function $R$ 
  of degree $\cO(\ln(l/\eps)^2)$
  so that replacing each $\srelu$ in $f$ with $R$ yields
  a function $g : [-1,+1]^d\to\R$ with
  \[
    \sup_{x\in[-1,+1]^d} |f(x) - g(x)| \leq \eps.
  \]
\end{lemma}

Combining \Cref{fact:descartes} and \Cref{fact:rational_net}
yields an intriguing corollary.

\begin{corollary}
  \label{fact:rat_sep}
  For every $k\geq 3$,
  there exists a function $f:\R\to\R$ computed by a rational
  network with $\cO(k)$ layers and $\cO(k)$ total nodes,
  each node invoking a rational activation of degree $\cO(k)$,
  such that every rational function $g:\R\to\R$
  with less than $2^{k-2}$ total terms in the numerator
  and denominator satisfies
  \[
    \int_{[0,1]} |f(x) - g(x)|\dif x \geq \frac 1 {128}.
  \]
\end{corollary}

The hard-to-approximate function $f$ is a rational network which
has a description of size $\cO(k^2)$.  Despite this, attempting
to approximate it with a rational function of the usual form requires
a description of size $\Omega(2^k)$.  Said another way:
even for rational
functions, there is a benefit to a neural network representation!

\subsection{Auxiliary results}

\begin{figure}[t]
  \centering
  \includegraphics[width=\columnwidth]{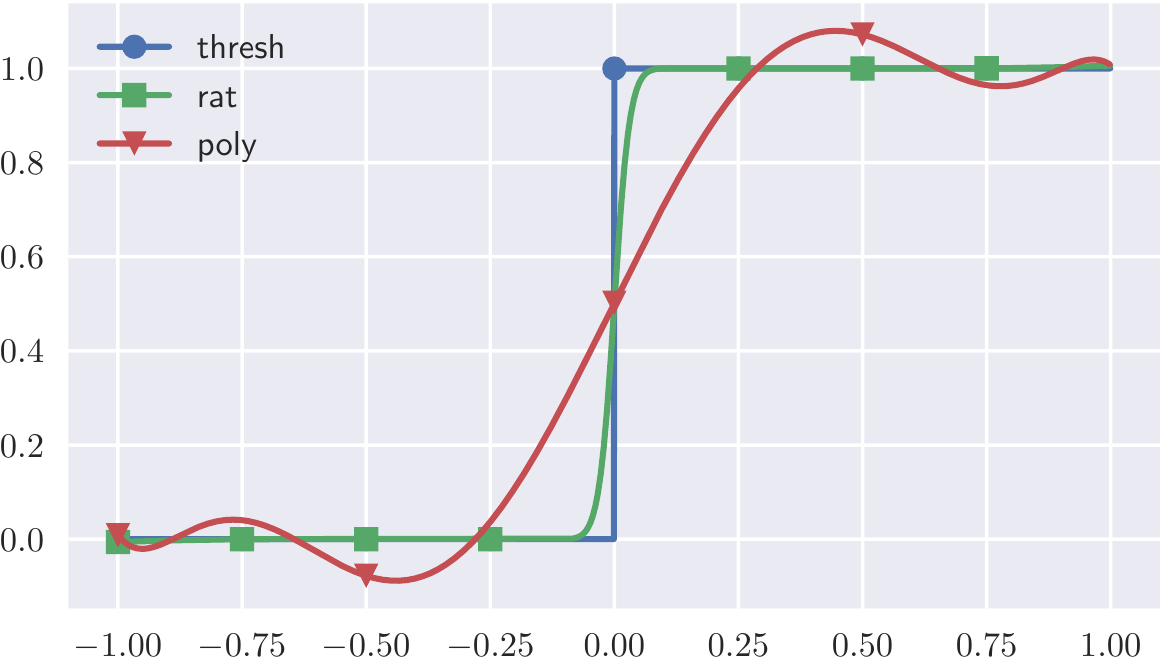}
  \caption{Polynomial and rational fit
    to the threshold function.}
  \label{fig:thresh}
\end{figure}

The first thing to stress
is that \Cref{fact:main} is impossible with polynomials:
namely, while it is true that ReLU networks can efficiently approximate
polynomials \citep{yarotsky,ohad__poly,srikant_poly},
on the other hand polynomials require
degree $\Omega(\poly(1/\eps))$, rather than $\cO(\poly(\ln(1/\eps)))$,
to approximate
a single ReLU, or equivalently the absolute value function
\citep[Chapter 4, Page 73]{petrusev_popov__rational}.

Another point of interest is the depth needed when converting
a rational function to a ReLU network.
\Cref{fact:main} is impossible if the depth is $o(\ln(1/\eps))$:
specifically, it is impossible to approximate the degree 1 rational
function $x\mapsto 1/x$ with size $\cO(\ln(1/\eps))$ but depth $o(\ln(1/\eps))$.

\begin{proposition}
  \label{fact:div:shallow}
  Set $f(x) := 1/x$, the reciprocal map.
  For any $\eps>0$ and ReLU network $g:\R\to\R$ with $l$
  layers and $m< (27648 \eps)^{-1/(2l)}/2$ nodes,
  \[
    \int_{[1/2,3/4]} |f(x) - g(x)| \dif x > \eps.
  \]
\end{proposition}

Lastly, the implementation of division in a ReLU network requires a
few steps, arguably the most interesting being a ``continuous
switch statement'', which computes reciprocals differently based on
the magnitude of the input.  The ability to compute switch statements
appears to be a fairly foundational operation available to neural
networks and rational functions
\citep[Theorem 5.2]{petrusev_popov__rational},
but is not available to polynomials (since otherwise they could
approximate the ReLU).

\subsection{Related work}

The results of the present work follow a long
line of work on the representation power of neural networks and
related functions.  The ability of ReLU networks to fit continuous
functions was no doubt proved many times, but it appears the earliest
reference is to Lebesgue \citep[Page 1]{newman__relu},
though of course results of this type are usually given much
more contemporary attribution \citep{cybenko}.
More recently, it has been shown that certain function classes
only admit succinct representations with many layers
\citep{mjt_easy_relu}.
This has been followed by proofs showing
the possibility for a depth 3 function to require exponentially many
nodes when rewritten with 2 layers \citep{ohad_nn_apx}.
There are also a variety of other result giving the ability of ReLU
networks to approximate various function classes
\citep{tensor_deep_repr_power,poggio__nn_repr}.

Most recently, a variety of works pointed out neural networks
can approximate polynomials, and thus smooth functions essentially
by Taylor's theorem
\citep{yarotsky,ohad__poly,srikant_poly}.
This somewhat motivates this present work, since
polynomials can not in turn approximate
neural networks with a dependence $\cO(\polylog(1/\eps))$:
they require degree $\Omega(1/\eps)$ even for a single ReLU.

Rational functions are extensively studied in the classical
approximation theory literature \citep{Lor2,petrusev_popov__rational}.
This literature draws close connections between rational functions
and \emph{splines} (piecewise polynomial functions), a connection
which has been used in the machine learning literature to
draw further connections to neural networks \citep{peter__rational}.
It is in this approximation theory
literature that one can find the following astonishing fact:
not only is it possible to approximate the absolute value function
(and thus the ReLU) over $[-1,+1]$ to accuracy $\eps>0$ with a
rational function of degree $\cO(\ln(1/\eps)^2)$ \citep{newman__relu},
but moreover the optimal rate is known
\citep{petrusev_popov__rational,zolotarev}!
These results form the basis of those
results here which show that rational functions can approximate ReLU
networks.
(Approximation theory results also provide other functions
(and types of neural networks) which rational functions can approximate
well, but the present work will stick to the ReLU for simplicity.)

An ICML reviewer revealed prior work which was embarrassingly overlooked by the author:
it has been known, \emph{since decades ago} \citep{ltf_division}, that neural networks using threshold
nonlinearities (i.e., the map $x\mapsto \1[x \geq 0]$)
can approximate division, and moreover the proof is similar to the
proof of part 1 of \Cref{fact:main}!  Moreover, other work on threshold networks
invoked Newman polynomials to prove lower bound about linear threshold
networks \citep{paturi_saks__rational}.
Together this suggests that not only the connections between rational
functions and neural networks are tight (and somewhat known/unsurprising), but also
that threshold networks and ReLU networks have perhaps more similarities than
what is suggested by the differing VC dimension bounds, approximation results,
and algorithmic results \citep{klivans_relu_alg}.

\subsection{Further notation}
\label{sec:notation}

Here is a brief description of the sorts of neural networks used
in this work.
Neural networks represent computation as a directed graph, where
nodes consume the outputs of their parents, apply a computation to them,
and pass the resulting value onward.  In the present work,
nodes take their parents' outputs $z$ and compute $\srelu(a^\top z + b)$,
where $a$ is a vector, $b$ is a scalar, and $\srelu(x) := \max\{0,x\}$;
another popular choice of nonlineary is the \emph{sigmoid}
$x\mapsto (1+\exp(-x))^{-1}$.
The graphs in the present work are acyclic and connected with a
single node lacking children designated as the univariate output,
but the literature contains many variations on all of these choices.

As stated previously, a rational function $f:\R^d\to\R$
is ratio of two polynomials.  Following conventions
in the approximation theory literature \citep{Lor2}, the denominator
polynomial will always be strictly positive.  The degree of a rational
function is the maximum of the degrees of its numerator and denominator.

\section{Approximating ReLU networks with rational functions}
\label{sec:rational_apx_relu}

\begin{figure}[t]
  \centering
  \includegraphics[width=\columnwidth]{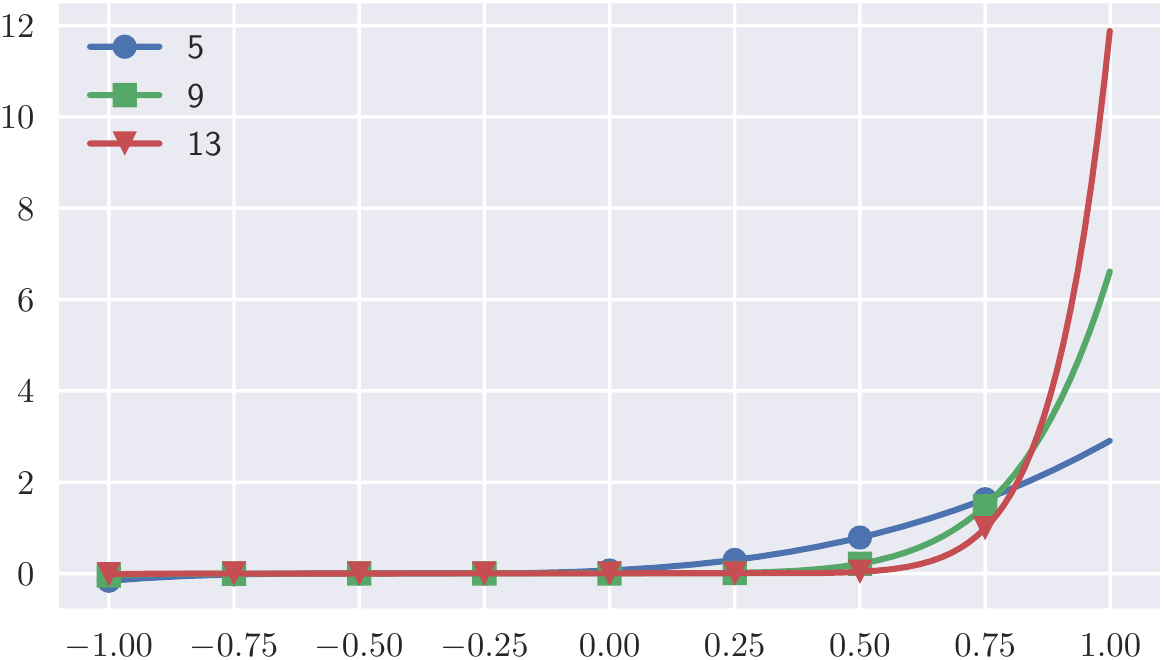}
  \caption{Newman polynomials of degree $5, 9, 13$.}
  \label{fig:newman}
\end{figure}

This section will develop the proofs of part 2 of \Cref{fact:main},
\Cref{fact:descartes}, \Cref{fact:rational_net},
and \Cref{fact:rat_sep}.

\subsection{Newman polynomials}

The starting point is a seminal result in the theory of rational
functions \citep{zolotarev,newman__relu}:
there exists a rational function of degree
$\cO(\ln(1/\eps)^2)$ which can approximate the absolute value function
along $[-1,+1]$ to accuracy $\eps>0$.  This in turn gives a way to
approximate the ReLU, since
\begin{equation}
  \srelu(x) = \max\{0,x\} = \frac {x+|x|}{2}.
  \label{eq:relu:abs}
\end{equation}

The construction here uses the \emph{Newman polynomials}
\citep{newman__relu}: given
an integer $r$, define
\[
  N_r(x) := \prod_{i=1}^{r-1} (x+ \exp(-i/\sqrt{r})).
\]
The Newman polynomials $N_5$, $N_9$, and $N_{13}$ are depicted in
\Cref{fig:newman}.  Typical polynomials in approximation theory, for
instance the Chebyshev polynomials, have very active oscillations;
in comparison, the Newman polynomials look a little funny,
lying close to 0 over $[-1,0]$, and quickly increasing monotonically
over $[0,1]$.  The seminal result of \citet{newman__relu} is that
\[
  \sup_{|x|\leq 1} \envert{ |x|
  - x \del{\frac{N_r(x)-N_r(-x)}{N_r(x)+N_r(-x)}}}
  \leq 3\exp(-\sqrt{r})/2.
\]
Thanks to this bound and \cref{eq:relu:abs},
it follows that the ReLU can be approximated to accuracy $\eps>0$
by rational functions of degree $\cO(\ln(1/\eps)^2)$.

(Some basics on Newman polynomials, as needed in the present work,
can be found in \Cref{sec:newman}.)

\subsection{Proof of \Cref{fact:rational_net}}

Now that a single ReLU can be easily converted to a rational function,
the next task is to replace every ReLU in a ReLU network with a rational
function, and compute the approximation error.
This is precisely the statement of \Cref{fact:rational_net}.

The proof of \Cref{fact:rational_net} is an induction on layers,
with full
details relegated to the appendix.  The key computation, however,
is as follows.
Let $R(x)$ denote a rational approximation to $\srelu$.
Fix a layer $i+1$,
and let $H(x)$ denote the multi-valued mapping computed by layer $i$,
and let $H_R(x)$ denote the mapping obtained by replacing each
$\srelu$ in $H$ with $R$.  Fix any node in layer $i+1$,
and let $x\mapsto\srelu(a^\top H(x) + b)$
denote its output as a function of
the input.
Then
\begin{align*}
  &\envert{ \srelu(a^\top H(x) + b) - R(a^\top H_R(x) + b) }
  \\
  &\leq
  \underbrace{
    \envert{ \srelu(a^\top H(x) + b) - \srelu(a^\top H_R(x) + b) }}_{\heartsuit}
  \\
  &\qquad+
  \underbrace{\envert{ \srelu(a^\top H_R(x) + b) - R(a^\top H_R(x) + b) }}_{\clubsuit}.
\end{align*}
For the first term $\heartsuit$, note since $\srelu$ is $1$-Lipschitz and
by H\"older's inequality that
\[
  \heartsuit \leq \envert{a^\top (H(x) - H_R(x))}
  \leq \|a\|_1 \|H(x) - H_R(x)\|_\infty,
\]
meaning this term has been reduced to the inductive hypothesis
since $\|a\|_1\leq 1$.  For the second term $\clubsuit$,
if $a^\top H_R(x) +b$ can be shown to lie in $[-1,+1]$
(which is another easy induction), then $\clubsuit$ is just the
error between $R$ and $\srelu$ on the same input.

\subsection{Proof of part 2 of \Cref{fact:main}}

It is now easy to find a rational function that approximates
a neural network, and to then bound its size.  The first step,
via \Cref{fact:rational_net}, is to replace each $\srelu$ with a
rational function $R$ of low degree (this last bit using
Newman polynomials).  The second step is to inductively
collapse the network into a single rational function.
The reason for the dependence on the number of nodes $m$
is that, unlike polynomials, summing rational functions
involves an increase in degree:
\[
  \frac {p_1(x)}{ q_1(x)}
  + \frac {p_1(x)}{ q_2(x)}
  = \frac {p_1(x) q_2(x) + p_2(x) q_1(x)}{q_1(x) q_2(x)}.
\]

\subsection{Proof of \Cref{fact:descartes}}

\begin{figure}[t]
  \centering
  \includegraphics[width=\columnwidth]{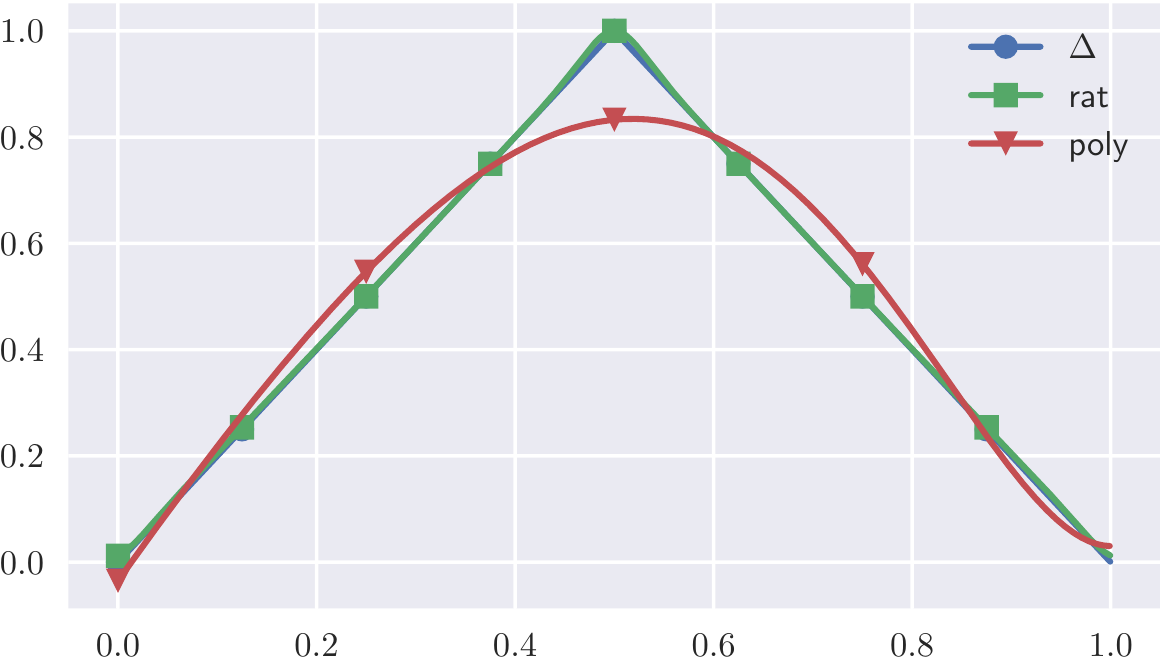}
    \caption{Polynomial and rational fit
    to $\Delta$.}
  \label{fig:triangle}
\end{figure}

The final interesting bit is to show that the dependence on
$m^l$ in part 2 of \Cref{fact:main} (where $m$ is the number of
nodes and $l$ is the number of layers) is tight.

Recall the ``triangle function''
\[
  \Delta(x) := \begin{cases}
    2x
    &x \in [0,1/2],
    \\
    2(1-x)
    &x \in (1/2,1],
    \\
    0
    &\textup{otherwise}.
  \end{cases}
\]
The $k$-fold composition $\Delta^k$ is a piecewise affine
function with $2^{k-1}$ regularly spaced peaks \citep{mjt_easy_relu}.
This function was demonstrated to be inapproximable by
shallow networks of subexponential size, and now it can be shown
to be a hard case for rational approximation as well.

Consider the horizontal line through $y=1/2$.  The function $\Delta^k$
will cross this line $2^k$ times.
Now consider a rational function $f(x) = p(x)/q(x)$.
The set of points where $f(x) = 1/2$ corresponds to points where
$2p(x) - q(x) = 0$.
A poor estimate for the number of zeros is simply the degree of
$2p - q$, however, since $f$ is univariate, a stronger tool  becomes
available: by Descartes' rule of signs, the number of zeros in $f-1/2$
is upper bounded by the number of terms in $2p - q$.

\section{Approximating rational functions with ReLU networks}
\label{sec:relu_apx_rational}

This section will develop the proof of part 1 of \Cref{fact:main},
as well as the tightness result in \Cref{fact:div:shallow}

\subsection{Proving part 1 of \Cref{fact:main}}

To establish part 1 of \Cref{fact:main}, the first step is to
approximate polynomials with ReLU networks, and the second
is to then approximate the division operation.

The representation of polynomials will be based upon constructions
due to \citet{yarotsky}.  The starting point
is the following approximation of the squaring function.

\begin{lemma}[name = {\citep{yarotsky}}]
  \label{fact:relu_square}
  Let any $\epsilon > 0$ be given.
  There exists $f : x \to [0,1]$,
  represented as a ReLU network with $\cO(\ln(1/\eps))$ nodes and layers,
  such that $\sup_{x\in [0,1]} |f(x) - x^2| \leq \eps $ and $f(0) = 0$.
\end{lemma}

Yarotsky's proof is beautiful and deserves mention.  The approximation
of $x^2$ is the function $f_k$, defined as
\[
  f_k(x) := x - \sum_{i=1}^k \frac{\Delta^i(x)}{4^i},
\]
where $\Delta$ is the triangle map from \Cref{sec:rational_apx_relu}.
For every $k$, $f_k$ is a convex, piecewise-affine interpolation between
points along the graph of $x^2$;
going from $k$ to $k+1$ does not adjust any of these
interpolation points, but adds a new set of $\cO(2^k)$ interpolation
points.

Once squaring is in place, multiplication comes via the polarization
identity $xy = ((x+y)^2 - x^2 - y^2)/2$.

\begin{lemma}[name = {\citep{yarotsky}}]
  \label{fact:relu_mul}
  Let any $\epsilon > 0$ and $B\geq 1$ be given.
  There exists $g(x,y) : [0,B]^2 \to [0,B^2]$,
  represented by a ReLU network with $\cO(\ln(B/\eps)$ nodes and layers,
  with
  \[
    \sup_{x,y\in [0,1]} |g(x,y) - xy| \leq \epsilon
  \]
  and $g(x,y) = 0$ if $x=0$ or $y=0$.
\end{lemma}

Next, it follows that ReLU networks can efficiently approximate
exponentiation thanks to repeated squaring.

\begin{lemma}
  \label{fact:relu_exp}
  Let $\eps \in (0,1]$ and positive integer $y$ be given.
  There exists $h : [0,1]\to[0,1]$,
  represented by a ReLU network with $\cO(\ln(y/\eps)^2)$ nodes and layers,
  with
  \[
    \sup_{x,y\in [0,1]} \envert{ h(x) - x^y } \leq \epsilon
  \]
\end{lemma}

With multiplication and exponentiation, a representation result
for polynomials follows.

\begin{lemma}
  \label{fact:relu_polynomial}
  Let $\eps \in (0,1]$ be given.
  Let $p: [0,1]^d\to[-1,+1]$ denote a polynomial with $\leq s$ monomials,
  each with degree $\leq r$ and scalar coefficient within $[-1,+1]$.
  Then there exists a function $q : [0,1]^d\to[-1,+1]$
  computed by a network of size
  $\cO\del{ \min\{ sr\ln(sr/\epsilon), sd\ln(dsr/\epsilon)^2 \} }$,
  which satisfies $\sup_{x\in[0,1]^d} |p(x) - q(x)| \leq \epsilon$.
\end{lemma}

The remainder of the proof now focuses on the division operation.
Since multiplication has been handled, it suffices to compute
a single reciprocal.

\begin{figure}[t]
  \centering
  \includegraphics[width=\columnwidth]{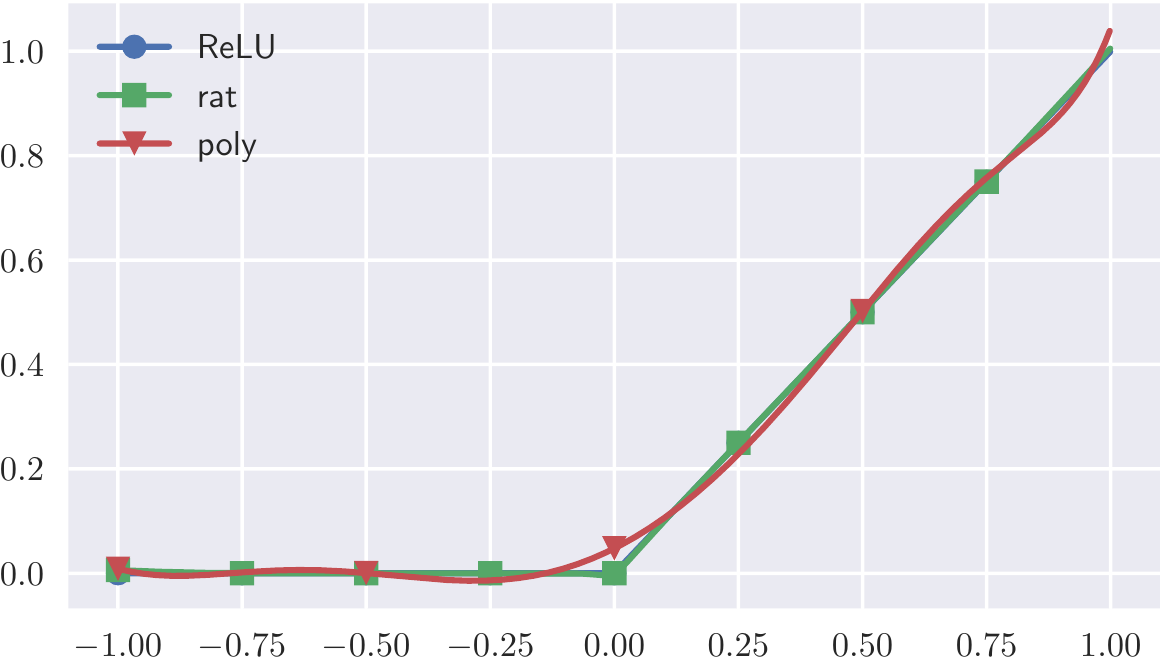}
  \caption{Polynomial and rational fit
    to $\srelu$.}
  \label{fig:relu}
\end{figure}

\begin{lemma}
  \label{fact:relu_invert}
  Let $\epsilon \in (0,1]$ and nonnegative integer $k$ be given.
  There exists a ReLU network $q : [2^{-k}, 1] \to [1,2^k]$,
  of size $\cO(k^2\ln(1/\eps)^2)$
  and depth $\cO(k^4 \ln(1/\eps)^3)$
  such that
  \[
    \sup_{x\in [2^{-k}, 1]} \envert{ q(x) - \frac 1 x} \leq \epsilon.
  \]
\end{lemma}

This proof relies on two tricks.  The first is to observe,
for $x\in (0,1]$, that
\[
  \frac 1 x
  = \frac 1 {1- (1-x)}
  = \sum_{i\geq 0} (1-x)^i.
\]
Thanks to the earlier development of exponentiation, truncating this summation
gives an expression easily approximate by a neural network as follows.

\begin{lemma}
  \label[lemma]{fact:relu_invert_nounity}
  Let $0 < a \leq b$ and $\eps > 0$ be given.
  Then there exists a ReLU network $q:\R\to\R$
  with $\cO(\ln(1/(a\eps))^2)$ layers
  and $\cO((b/a)\ln(1/(a\eps))^3)$ nodes
  satisfying
  \[
    \sup_{x\in[a,b]} \envert{ q(x) - \frac 1 x } \leq 2\eps.
  \]
\end{lemma}

Unfortunately, \Cref{fact:relu_invert_nounity} differs
from the desired statement \Cref{fact:relu_invert_nounity}:
inverting inputs lying within $[2^{-k},1]$ requires $\cO(2^k\ln(1/\eps)^2)$ nodes
rather than $\cO(k^4 \ln(1/\eps)^3)$!

To obtain a good estimate with only $\cO(\ln(1/\eps))$ terms of the summation, it
is necessary for the input to be $x$ bounded below by a positive constant
(not depending on $k$).
This leads to the second trick
(which was also used by \citet{ltf_division}!).

Consider, for positive constant $c>0$, the expression
\[
  \frac 1 x
  = \frac {c}{1 - (1- cx)}
  = c \sum_{i\geq 0} (1-cx)^i.
\]
If $x$ is small, choosing a larger $c$ will cause this summation
to converge more quickly.  Thus, to compute $1/x$ accurately
over a wide range of inputs,
the solution here is to multiplex approximations
of the truncated sum for \emph{many} choices of $c$.   In order to only
rely on the value of one of them, it is possible to encode a large
``switch'' style statement in a neural network.  Notably,
rational functions can also representat switch statements
\citep[Theorem 5.2]{petrusev_popov__rational},
however polynomials can not (otherwise they could
approximate the ReLU more efficiently, seeing as it is a switch
statement of 0 (a degree 0 polynomial) and $x$ (a degree 1 polynomial).

\begin{lemma}
  \label[lemma]{fact:relu_part_unity}
  Let $\eps>0$,
  $B\geq 1$,
  reals $a_0\leq a_1 \leq \cdots \leq a_n \leq a_{n+1}$ and a
  function $f: [a_0,a _{n+1}] \to \R$ be given.  Moreover, suppose
  for $i \in \{1,\ldots,n\}$, there exists a ReLU network $g_i:\R\to\R$
  of size $\leq m_i$ and depth $\leq k_i$ with
  $g_i\in[0,B]$ along $[a_{i-1},a_{i+1}]$
  and
  \[
    \sup_{x\in [a_{i-1}, a_{i+1}]} |g_i(x) - f| \leq \eps.
  \]
  Then there exists a function $g:\R\to\R$ computed by a ReLU network
  of size $\cO\del{n\ln(B/\eps) + \sum_i m_i}$
  and depth $\cO\del{\ln(B/\eps) + \max_i k_i}$
  satisfying
  \[
    \sup_{x\in [a_1, a_n]} |g(x) - f(x)| \leq 3\eps.
  \]
\end{lemma}

\subsection{Proof of \Cref{fact:div:shallow}}

It remains to show that shallow networks have a hard time approximating
the reciprocal map $x\mapsto 1/x$.

This proof uses the same scheme as various proofs in
\citep{mjt_nn},
which was also followed in more recent works
\citep{yarotsky,ohad__poly}: the idea is to first
upper bound the number of affine pieces in ReLU networks of a certain
size, and then to point out that each linear segment must make
substantial error on a curved function, namely $1/x$.

The proof is fairly brute force, and thus relegated to the appendices.

\section{Summary of figures}

Throughout this work, a number of figures were presented to show
not only the astonishing approximation properties of rational functions,
but also the higher fidelity approximation achieved by both ReLU networks
and rational functions as compared with polynomials.  Of course,
this is only a qualitative demonstration, but still lends some intuition.

In all these demonstrations, rational functions and polynomials have
degree 9 unless otherwise marked.  ReLU networks have two hidden layers
each with 3 nodes.  This is not exactly apples to apples (e.g., the
rational function has twice as many parameters as the polynomial),
but still reasonable as most of the approximation literature
fixes polynomial and rational degrees in comparisons.

\Cref{fig:recip} shows the ability of all three classes to approximate
a truncated reciprocal.  Both rational functions and ReLU networks have
the ability to form ``switch statements'' that let them approximate
different functions on different intervals with low complexity
\citep[Theorem 5.2]{petrusev_popov__rational}.
Polynomials lack this ability; they can not even approximate
the ReLU well, despite it being low degree polynomials on two separate
intervals.

\Cref{fig:thresh} shows that rational functions can fit the threshold
function errily well; the particular rational function used
here is based on using Newman polynomials to approximate
$(1 + |x|/x)/2$ \citep{newman__relu}.

\Cref{fig:newman} shows Newman polynomials $N_5$, $N_9$, $N_{13}$.
As discussed in the text, they are unlike orthogonal polynomials,
and are used in all rational function approximations except
\Cref{fig:recip}, which used a least squares fit.

\begin{figure}[t]
  \centering
    \includegraphics[width=\columnwidth]{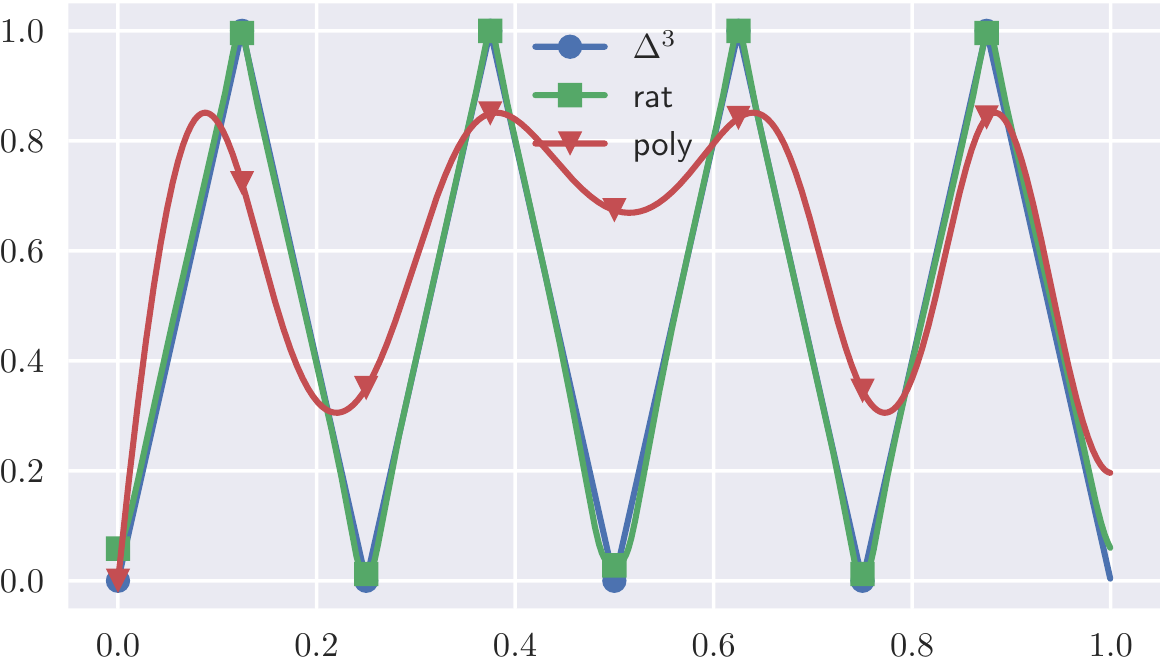}
  \caption{Polynomial and rational fit
    to $\Delta^3$.}
  \label{fig:triangle_3}
\end{figure}

\Cref{fig:triangle} shows that rational functions (via the Newman
polynomials) fit $\Delta$ very well, whereas polynomials have trouble.
These errors degrade sharply after recursing, namely when
approximating $\Delta^3$ as in \Cref{fig:triangle_3}.

\Cref{fig:relu} shows how polynomials and rational functions fit
the ReLU, where the ReLU representation, based on Newman polynomials,
is the one used in the proofs here.  Despite the apparent
slow convergence of polynomials in this regime, the polynomial fit
is still quite respectable.

\section{Open problems}

There are many next steps for this and related results.

\begin{enumerate}
  \item
    Can rational functions, or some other approximating class,
    be used to more tightly bound the generalization properties
    of neural networks?  Notably, the VC dimension of sigmoid
    networks uses a conversion to polynomials
    \citep{anthony_bartlett_nn}.

  \item
    Can rational functions, or some other approximating class,
    be used to design algorithms for training neural networks?
    It does not seem easy to design reasonable algorithms
    for minimization over rational functions; if this is fundamental
    and moreover in contrast with neural networks,
    it suggests an algorithmic benefit of neural networks.

  \item
    Can rational functions, or some other approximating class,
    give a sufficiently refined complexity estimate of neural
    networks which can then be turned into a regularization scheme
    for neural networks?
\end{enumerate}

\section*{Acknowledgements}
The author thanks Adam Klivans and Suvrit Sra for stimulating conversations.
Adam Klivans and the author both thank Almare Gelato Italiano, in downtown Berkeley,
for necessitating further stimulating conversations,
but now on the topic of health and exercise.
Lastly, the author thanks the University of Illinois, Urbana-Champaign,
and the Simons Institute in Berkeley,
for financial support during this work.

\bibliography{bib}
\bibliographystyle{icml2017}

\appendix

\onecolumn

\section{Deferred material from \Cref{sec:rational_apx_relu}}

This section collects technical material omitted from
\Cref{sec:rational_apx_relu}.  The first step is to fill in some
missing details regarding Newman polynomials.

\subsection{Newman polynomials}
\label{sec:newman}

Define the \emph{Newman polynomial} \citep{newman__relu}
\begin{equation}
  N_r(x) := \prod_{i=1}^{r-1} (x + \alpha_r^i)
  \qquad
  \textup{where }
  \alpha_r := \exp(-1/\sqrt{r}).
\end{equation}
Define $A_r(x)$, the Newman approximation to $|x|$, as
\[
  A_r(x) := x \del{ \frac {N_r(x) - N_r(-x)}{N_r(x) + N_r(-x)} }.
\]

\begin{lemma}[name = {\citet{newman__relu}}]
  \label{fact:newman}
  Suppose $r\geq 5$.
  \begin{itemize}
    \item
      $N_r(x) + N_r(-x) > 0$;
      in particular, $A_r$ is well-defined over $\R$.
    \item
      Given any $b \geq 1$,
      \[
        \sup_{x \in[-b,+b]} \envert{
          bA_r(x/b) - |x|
        }
        \leq 3b\exp(-\sqrt{r}).
      \]
  \end{itemize}
\end{lemma}
\begin{proof}
  \begin{itemize}
    \item
      If $x=0$, then $N_r(-x) = N_r(x) = \prod_{i=1}^{r-1} \alpha_r^{i} > 0$.
      Otherwise $x > 0$, and note for any $i\in\{1,\ldots,r-1\}$ that
      \begin{itemize}
        \item
        $x\in (0, \alpha_r^i]$ means $|x - \alpha_r^i| = \alpha_r^i - x < \alpha_r^i + x$,
      \item
        $x > \alpha_r^i$ means $|x-\alpha_r^i| = x - \alpha_r^i < x + \alpha_r^i$.
    \end{itemize}
    Together, $|x-\alpha_r^i| < x + \alpha_r^i$, and
    \[
      N_r(x)
      = \prod_{i=1}^{r-1} (x + \alpha_r^i)
      > \prod_{i=1}^{r-1} |x - \alpha_r^i|
      = \left| \prod_{i=1}^{r-1} (x - \alpha_r^i) \right|
      = |N_r(-x)|.
    \]
    Since $N_r(x) > 0$ when $x >0$, thus $N_r(x) + N_r(-x) > N_r(x) - |N_r(x)| = 0$.

    Lastly, the case $x < 0$ follows from the case $x >0$ since $x \mapsto N_r(x) + N_r(-x)$ is even.

  \item
    For any $x \in [-b,+b]$,
    \[
      | |x| - bA_r(x/b) |
      = \left| b \left( | x/b| - A_r(x/b) \right) \right|
      = b \left| x/b - A_r(x/b) \right|
      \leq 3b \exp(-\sqrt{r}),
    \]
    where the last step was proved by Newman \citep[Theorem 7.3.1]{Lor2}.
  \end{itemize}
\end{proof}

Finally, define
\begin{align*}
  \tilde R_{r,b}(x)
  &:= R_r(x;b)
  := \frac {x + bA_r(x/b)}{2},
  \\
  \eps_{r,b}
  &:= 3 \exp(-\sqrt{r})/2,
  \\
  R_{r,b}(x)
  &:= (1-2\eps_{r,b}) \tilde R_{r,b}(x) + b\eps_{r,b}.
\end{align*}
\begin{lemma}
  \label{fact:relu_newman}
  If $r\geq 5$ and $b \geq 1$ and $\eps_{r,b} \leq 1/2$,
  then $R_{r,b}$ is a degree-$r$ rational function over $\R$,
  and
  \begin{align*}
    \sup_{x\in[-b,+b]} \envert{ \srelu(x) - \tilde R_{r,b}(x)}
    &\leq b\eps_{r,b},
    \\
    \sup_{x\in[-b,+b]} \envert{ \srelu(x) - R_{r,b}(x)}
    &\leq 3b\eps_{r,b}.
  \end{align*}
  If $\eps_{r,b} \leq 1$, then $R_{r,b} \in [0,b]$ along $[-b,+b]$.
\end{lemma}
\begin{proof}
  Let $r,b$ be given, and for simplicity omit the various subscripts.
  The denominator of $\tilde R$
  is positive over $\R$ by \Cref{fact:newman}.
  Now fix $x\in [-b,+b]$.
  Using the second part of \Cref{fact:newman},
  \[
    \left|
      \srelu(x) - \tilde R(x)
    \right|
    =
    \left|
      \frac {x + |x|}{2} - \frac {x + bA(x/b)}{2}
    \right|
    =
    \frac 1 2
    \left|
      |x| - bA(x/b)
    \right|
    \leq
    3b\exp(-\sqrt{r})/2 = b\eps.
  \]
  Next, note that $\tilde R \in [-b\eps, b(1+\eps)]$:
  \begin{align*}
    \tilde R(x) \leq \srelu(x) + b\eps \leq b(1+\eps),
    \qquad
    \tilde R(x) \geq \srelu(x) - b\eps \geq -b\eps.
  \end{align*}
  Thus
  \begin{align*}
    \envert{ \srelu(x) - R(x) }
    &\leq \envert{ \srelu(x) - \tilde R(x) }
    + \envert{ \tilde R(x) - R(x)}
    \\
    &\leq
    b\eps + 0 + 2\eps\envert{ \tilde R(x) - b/2 }
    \\
    &\leq 3b\eps.
  \end{align*}
  Moreover
  \begin{align*}
    R(x)
    &= (1-2\eps)\tilde R(x) + b\eps
    \geq (1-2\eps)(-b\eps) + b\eps
    \geq 0,
    \\
    R(x)
    &\leq (1-2\eps)b(1+\eps) + b\eps
    \leq b.
  \end{align*}

\end{proof}

\subsection{Remaining deferred proofs}

The details of converting a ReLU network into a rational network
are as follows.

\begin{lemma}
  Let $f:\R^d\to\R$ be represented by a ReLU network with $\leq l$
  layers,
  and with each node computing a map
  $z\mapsto \srelu(a^\top z + b)$ where $\|a\|_1 + |b| \leq 1$.
  Then for every $\eps > 0$
  there exists a function $g:\R^d\to\R$ with
  $|g(x) - f(x)| \leq \eps$ for $\|x\|_\infty \leq 1$
  where $g$ is obtained from $f$ by replacing each ReLU
  with an $r$-rational function with $r = \cO(\ln(1/\eps)^2)$.
\end{lemma}
\begin{proof}[Proof of \Cref{fact:rational_net}]
  This construction will use the Newman-based approximation $R:=R_{r,b}$
  to $\srelu$ with degree $\cO(\ln(l/\eps)^2)$.
  By \Cref{fact:relu_newman}, this degree suffices to guarantee
  $R(x) \in [0,1]$ and $|R(x) - \srelu(x)| \leq \eps/l$
  for $|x|\leq 1$.

  First note, by induction on layers, that the output of
  every node has absolute value at most 1.
  The base case is the inputs themselves, and thus the statement
  holds by the assumption $\|x\|_\infty \leq 1$.
  In the inductive step, consider any node $z\mapsto R(a^\top z + b)$,
  where $z$ is the multivariate input to this node.
  By the inductive hypothesis, $\|z\|_\infty \leq 1$,
  thus
  \[
    |a^\top z + b| \leq \|a\|_1 \|z\|_\infty + |b| \leq 1.
  \]
  As such, $R(a^\top z + b) \in [0,1]$.

  It remains to prove the error bound.  For any node, if $h:\R^d\to\R$
  denote the function (of the input $x$) compute by this node,
  then let $h_R$ denote the function obtained by replacing
  all ReLUs with $R$.  It will be shown that every node in layer $i$
  has $|h_R(x) - h(x)|\leq i\eps/l$ when $\|x\|_\infty \leq 1$.
  The base case is the inputs themselves, and thus there is no
  approximation error, meaning the bound holds with error
  $0\leq 1\cdot\eps/l$.
  Now consider any node in layer $i+1$ with $i\geq 0$, and
  suppose the claim holds for nodes in layers $i$ and lower.
  For convenience, let $H$ denote the multivalued map computed by
  the previous layer, and $H_R$ denote the multivalued map obtained by
  replacing all activations in earlier layers with $R$.
  Since $\srelu$ is $1$-Lipschitz,
  and since the earlier boundedness property grants
  \[
    \envert{ a^\top H_R(x) + b }
    \leq \|a\|_1 \|H_R(x)\|_\infty + |b|
    \leq 1,
  \]
  then
  \begin{align*}
    |h(x) - h_R(x)|
    &= \envert{ \srelu(a^\top H(x) + b) - R(a^\top H_R(x) + b) }
    \\
    &\leq \envert{ \srelu(a^\top H(x) + b) - \srelu(a^\top H_R(x) + b) }
    + \envert{ \srelu(a^\top H_R(x) + b) - R(a^\top H_R(x) + b) }
    \\
    &\leq \envert{a^\top H(x) - a^\top H_R(x)}
    + \eps/l
    \\
    &
    \leq \|a\|_1 \|H-H_R\|_\infty + \eps/l
    \\
    &\leq (i+1)\eps/l.
  \end{align*}
\end{proof}

Next, collapsing a rational network down into a single rational
function is proved as follows.

\begin{lemma}
  \label{fact:rat_net_collapse}
  Let $f:\R^d\to\R$ be a rational network
  with $\leq m$ nodes in each of $\leq l$ layers,
  and the activation function has degree $r$.
  Then the rational function obtained by collapsing
  $f$ has degree at most $(rm)^l$.
\end{lemma}
\begin{proof}
  Throughout this proof, let $R$ denote the rational activation
  function at each node, and write $R(x) = p(x) / q(x)$ where
  $p$ and $q$ are polynomials of degree at most $r$.
  The proof establishes, by induction on layers,
  that the nodes of layer $i$ compute rational
  functions of degree at most $(rm)^i$.
  The base case is layer $1$, where each node
  computes a rational function of degree $r \leq rm$.
  For the case of layer $i > 1$, fix any node,
  and denote its computation by $h(x) = R(\sum_{j=1}^n a_j g_j(x) + b)$,
  where $n \leq m$ and $g_j = p_j/q_j$
  is a rational function of degree at most $(rm)^{i-1}$.
  Note
  \begin{align*}
    \deg\del{\sum_j \frac {a_j  p_j(x)}{q_j(x)} + b}
    &=
    \deg\del{\frac{b \prod_j q_j(x) + \sum_j a_j p_j(x) \prod_{k \neq j} q_k(x)}{\prod_j q_j(x)}}
    \\
    &\leq
    m(mr)^{i-1}.
  \end{align*}
  the map $f := \sum_j a_j g_j + b$ is rational of degree $m(mr)^{i-1}$.
  Let $p_f$ and $q_f$ denote its numerator and denominator.
  Since $R$ is univariate, its numerator $p$ and denominator $q$
  have the form $p(x) := \sum_{j\leq r} c_j x^j$
  and $q(x) \sum_{j\leq r}d_j x^j$.
  Thus, using the fact that $q > 0$,
  \begin{align*}
    \deg(h(x)) = \deg(R(f(x)))
    &= \deg\del{
      \frac
      {\sum_{j\leq r} c_j (p_f(x)/q_f(x))^j}
      {\sum_{j\leq r} d_j (p_f(x)/q_f(x))^j}
      \del{ \frac{ q_f(x)^r }{ q_f(x)^r } }
    }
    \\
    &= \deg\del{
      \frac
      {\sum_{j\leq r} c_j p_f(x)^j q_f(x))^{r-j}}
      {\sum_{j\leq r} d_j p_f(x)^j q_f(x))^{r-j}}
    }
    \leq rm(rm)^{i-1} = (rm)^i.
  \end{align*}
\end{proof}

The proof of part 2 of \Cref{fact:main} now follows by combining
\Cref{fact:rational_net,fact:relu_newman,fact:rat_net_collapse}.

The last piece is a slighly more detailed account
of \Cref{fact:descartes}.

\begin{proof}[Proof of \Cref{fact:descartes}]
  Let $\Delta:\R\to\R$ denote the triangle function from
  \citep{mjt_easy_relu}:
  \[
    \Delta(x) := \begin{cases}
      2x & x \in[0,1/2],
      \\
      2(1-x) & x \in(1/2,1],
      \\
      0
      &\textup{otherwise.}
    \end{cases}
  \]
  Define the target function $f = \Delta^k$, which as in
  \citep{mjt_easy_relu}
  has $2^k$ regular-spaced crossings of 0 along $[0,1]$,
  and can be written as a network with $2k$ layers, each with $\leq 2$
  nodes.

  Next consider the rational function $g$.  As in the text, it is
  necessary to count the zeros of $g - 1/2$
  (the case $g=1/2$ is trivial).
  Writing $g = p/q$, equivalently this means the zeros of $2p - q$.
  Since $p$ and $q$ together have $\leq 2^{k-2}$ terms,
  by Descartes' rule of signs, $g$ crosses $1/2$ at most $2^{k-2}$
  times along $(0,1]$.
  Therefore, following a similar calculation to the proof in
  \citep[Proof of Theorem 1.1]{mjt_nn},
  \begin{align*}
    \int_{(0,1]} |f(x) - g(x)| \dif x
    \geq
    \frac 1 {32} \left( 1 - \frac {2(2^{k-2})}{2^k}\right)
    = \frac 1 {64}.
  \end{align*}
\end{proof}

\section{Deferred material from \Cref{sec:relu_apx_rational}}

\subsection{Towards the proof of part 1 of \Cref{fact:main}}

To start, the lemmas due to Yarotsky are slightly adjusted to
clip the range to $[0,1]$.

\begin{proof}[Proof of \Cref{fact:relu_square}]
  Inspecting Yarotsky's proof, the construction provides $g(x)$ with
  $g(0) = 0$ and $\sup_{x\in[0,1]} |g(x) - x^2| \leq \epsilon$.
  To provide the desired $f$, it suffices to define
  $f(x) = \srelu(g(x)) - \srelu(g(x) - 1)$.
\end{proof}

\begin{proof}[Proof of \Cref{fact:relu_mul}]
  First suppose $B = 1$,
  let $f$ be as in \Cref{fact:relu_square} at resolution $\eps/8$,
  and define $h$ via the polarization identity (as in Yarotsky's proof):
  \[
    h(x,y) = 2 (f(x/2+y/2) - f(x/2) - f(y/2))
  \]
  (where $x/2$ appears since $f$ has domain $[0,1]^2$).
  Since $f(0) = 0$,
  \[
    h(x,0) = 2 (f(x/2) - f(x/2) - 0) = 0,
    \qquad
    h(0,y) = 2 (f(y/2) - 0 - f(y/2)) = 0.
  \]
  Moreover, for any $x,y\in [0,1]$
  \begin{gather*}
      h(x,y) - xy
      \leq 2 \del{ (x/2+y/2)^2 + \eps/8 - x^2/4 + \eps/8 - y^2/4 + \eps/8} - xy \leq xy + \eps,
      \\
      h(x,y) - xy
      \geq 2 \del{ (x/2+y/2)^2 - \eps/8 - x^2/4 - \eps/8 - y^2/4 - \eps/8} - xy \leq xy - \eps.
  \end{gather*}
  Finally, set $\tilde g(x,y) := \srelu(h(x,y)) - \srelu(h(x,y) - 1)$,
  which preserves the other properties.

  Now consider the case $B \geq 1$, and set $g(x,y) = B^2 g(x/B,y/B)$.
  Then $(x,y)\in[0,B]^2$ implies
  \[
    \envert{ g(x,y) - xy }
    = B^2 \envert{ \tilde g(x/B, y/B) - (x/B)(y/B) }
    \leq \eps B^2,
  \]
  and $g \in [0,B^2]$ over $[0,B]^2$
  since $\tilde g \in [0,1]$ over $[0,1]^2$.
\end{proof}

The full details for the proof of fast exponentiation are as follows.

\begin{proof}[Proof of \Cref{fact:relu_exp}]
  This proof constructs a network implementing the \emph{russian peasant algorithm for exponentiation}:
  \begin{enumerate}
    \item
      Set $v := 1$.

    \item
      For $b \in \textup{bits-ltr}(y)$ (the bits of $y$ from left to right):
      \begin{enumerate}
        \item
          Set $v := v^2$.
        \item
          If $b=1$, set $v := vx$.
      \end{enumerate}
  \end{enumerate}
  For example,
  \[
    x^{10101_2} = ((((1^2\cdot x)^2)^2\cdot x)^2)^2 \cdot x = x^{2^4} \cdot x^{2^2} \cdot x.
  \]

  The two lines in the inner loop will use the squaring function $f$ from \Cref{fact:relu_square}
  and the multiplication function $g$ from \Cref{fact:relu_mul},
  each with accuracy $c\epsilon$ where $c:= 1/y^2$.
  At the end, the network returns $\srelu(v) - \srelu(v-1)$ to ensure the output lies in $[0,1]$; this procedure can not increase the error.
  Since the loop is invoke $\cO(\ln(y))$ times and each inner loop requires a network of size $\cO(\ln(1/(c\epsilon))) = \cO(\ln(y/\epsilon))$,
  the full network has size $\cO(\ln(y/\epsilon)^2)$.

  It remains to show that the network computes a function $h$ which satisfies
  \[
    h(x) = x^y.
  \]
  Let $z_j$ denote the integer corresponding to the first $j$ bits of $y$ when read left-to-right;
  it will be shown by induction (on the bits of $y$ from left to right) that, at the end of the $j$th invocation of the loop,
  \[
    \envert{ v - x^{z_j} } \leq z_j^2 c\epsilon.
  \]
  This suffices to establish the claim since then $|v - x^y | \leq y^2 c\epsilon = \epsilon$.

  For the base case, consider $j=0$; then $v=1 = x^{z_0} = x^0$ as desired.
  For the inductive step, let $w$ denote $v$ at the end of the previous iteration,
  whereby the inductive hypothesis grants
  \[
    \envert{ w - x^{z_{j-1}} } \leq z_{j-1}^2 c\epsilon.
  \]
  The error after the approximate squaring step can be upper bounded as
  \begin{align*}
    f(w) - x^{2z_{j-1}}
    &\leq
    \del{ f(w) - w^2 } + \del{ w^2 - x^{2z_{j-1}} }
    \\
    &\leq
    c \eps + \del{ (x^{z_{j-1}} + z_{j-1}^2 c\epsilon)^2 - x^{2z_{j-1}} }
    \\
    &\leq
    c \eps +  2 z_{j-1}^2 c\epsilon + z_{j-1}^4 c^2 \epsilon^2
    \\
    &\leq
    c \eps +  2 z_{j-1}^2 c\epsilon + z_{j-1}^2 c \epsilon
    \\
    &\leq
    (2 z_{j-1})^2 c\eps.
  \end{align*}
  The reverse inequality is proved analogously, thus
  \[
    \envert{ f(w) - x^{2z_{j-1}} }
    \leq
    (2 z_{j-1})^2 c\eps.
  \]
  If the bit $b$ in this iteration is 0, then $2z_{j-1} = z_j$ and the proof for this loop iteration
  is complete.  Otherwise $b=1$, and
  \begin{align*}
    v - x^{z_j}
    &=
    g(f(w), x) - x^{2z_{j-1} + b}
    \\
    &\leq
    xf(w) + c\epsilon - x^{2z_{j-1} + b}
    \\
    &\leq
    \del{ (2z_{j-1})^2 + 1}c\eps
    \\
    &\leq
    (z_j)^2c\eps.
  \end{align*}
  The proof of the reverse inequality is analogous, which establishes the desired error bound on $v$ for this loop iteration.
\end{proof}

Using the preceding exponentiation lemma,
the proof of polynomial approximation is as follows.

\begin{proof}[Proof of \Cref{fact:relu_polynomial}]
  It will be shown momentarily that a single monomial term can be approximating
  to accuracy $\epsilon/s$ with a network of size
  $\cO\del{ \min\{ r\ln(sr/\epsilon),  d\ln(dsr/\epsilon)^2 \} }$.
  This implies the result by summing $\leq s$ monomials comprising a polynomial,
  along with their errors.

  For a single monomial, here are two constructions.
  \begin{itemize}
    \item
      One approach is to product together $\leq r$ individual variables (and lastly multiple by a fixed
      scalar coefficient), with no concern of the multiplicities of individual variables.
      To this end, let $(y_1,\ldots,y_k)$ with $s\leq r$ denote coordinates of the input variable
      so that $\prod_{i=1}^k y_i$ is the desired multinomial.
      Let $g$ denote multiplication with error $\eps_0 := \eps/(rs)$ as provided by \Cref{fact:relu_mul}.
      The network will compute $\alpha g_i(y)$,
      where $\alpha\in[-1,+1]$ is the scalar coefficient on the monomial, and
      $g_i$ is recursively defined as
      \[
        g_1(y) = y_1,
        \qquad
        g_{i+1}(y) := f(y_{i+1}, g_i(y))
      \]
      It is established by induction that
      \[
          \envert{ g_i(y) - \prod_{j=1}^i y_j } \leq j \eps_0
      \]
      The base case is immediate since $g_1(y) = y_1 = \prod_{j=1}^1 y_j$.
      For the inductive step,
      \[
       g_{i+1}(y) - \prod_{j=1}^{i+1} y_j
       = f(y_{i+1}, g_i(y)) - \prod_{j=1}^{i+1} y_j
       \leq y_{i+1} g_i(y) + \eps_0 - \prod_{j=1}^{i+1} y_j
       \leq y_{i+1} (i \eps_0 + \prod_{j=1}^i y_j) + \eps_0 - \prod_{j=1}^{i+1} y_j
       \leq (i+1)\eps_0,
      \]
      and the reverse inequality is proved analogously.

    \item
      Alternatively, the network uses the fast exponentiation routine from \Cref{fact:relu_exp},
      and then multiplies together the terms for individual coordinates.
      In particular, the exponentiation for each coordinate with accuracy $\eps_1 := \eps/(ds)$
      requires a network of size $\cO(\ln(r/\eps_1)^2)$.
      By an analysis similar to the preceding construction, multiplying $\leq d$ such networks
      will result in a network approximating the monomial with error $\eps/s$ and
      size $\cO(d \ln(r/\eps_1)^2)$.
  \end{itemize}
\end{proof}

Next, the proof that ReLU networks can efficiently compute
reciprocals, namely \Cref{fact:relu_invert}.
As stated in the text, it is first necessary to establish
\Cref{fact:relu_invert_nounity}, which gives computes reciprocals
at a choice of magnitude, and then \Cref{fact:relu_part_unity},
which combines these circuits across scales.

\begin{proof}[Proof of \Cref{fact:relu_part_unity}]
  For each $i\in \{1,\ldots,n\}$,
  define the function
  \[
    p_{i}(z) := \begin{cases}
      \frac{z - a_{i-1}}{a_i - a_{i-1}}
      &
      z \in [a_{i-1}, a_i],
      \\
      \frac{a_{i+1}-z}{a_{i+1} - a_{i}}
      &
      z \in (a_i, a_{i+1}],
      \\
      0& \textup{otherwise}.
    \end{cases}
  \]
  The functions $(p_i)_{i=1}^n$ have the following properties.
  \begin{itemize}
    \item
      Each $p_i$ can be represented by a ReLU network with three nodes
      in 2 layers.
    \item
      For any $x \in [a_1, a_n]$, there exists $j \in \{1,\ldots,n\}$
      so that $i \in \{j,j+1\}$ implies $p_i(x) \geq 0$ and
      $i \not \in \{j,j+1\}$ implies $p_i(x) = 0$.
      Indeed, it suffices to let $j$ be the smallest element
      of $\{1,\ldots,n-1\}$.
      satisfying $x \in [a_j,a_{j+1}]$.
    \item
      For any $x \in [a_1,a_n]$, $\sum_{i=1}^n p_i(x) = 1$.
  \end{itemize}
  The family $(p_i)_{i=1}^n$ thus forms a \emph{partition of unity}
  over $[a_1,a_n]$, moreover with the property that at most two elements,
  necessarily consecutive, are nonzero at any point in the interval.

  Let $h:[0,B]^2 \to [0,B]$ be a uniform $\eps$-approximation via ReLU
  networks to the multiplication map $(x,y)\mapsto xy$; by
  \Cref{fact:relu_mul}, $h$ has $\cO(\ln(B/\eps))$ nodes and layers,
  and moreover the multiplication is exact when either input is 0.
  Finally, define $g:\R\to\R$ as
  \[
    g(x) := \sum_{i=1}^n h(p_i(x), g_i(x)).
  \]
  By construction, $g$ is a ReLU network with
  $\cO(\ln(B/\eps) + \max_i k_i)$ layers
  and $\cO(n\ln(B/\eps) + \sum_i m_i)$ nodes.

  It remains to check the approximation properties of $g$.
  Let $x \in [a_1,a_n]$ be given,
  and set $j := \min \cbr{ j \in \{1,n-1\} : x \in [a_j, a_{j+1}] }$.
  Then
  \begin{align*}
    \envert{
      f(x)
      - g(x)
    }
    &=
    \envert{
      f(x)
      - h(p_j(x), g_j(x))
      - h(p_{j+1}(x), g_{j+1}(x))
    }
    \\
    &\leq
    \envert{
      f(x)
      - p_j(x) g_j(x)
      - p_{j+1}(x) g_{j+1}(x)
    }
    \\
    &\qquad
    +
    \envert{
      p_j(x) g_j(x)
      - h(p_j(x), g_j(x))
    }
    +
    \envert{
      p_{j+1}(x) g_{j+1}(x)
      - h(p_{j+1}(x), g_{j+1}(x))
    }
    \\
    &\leq
    p_j(x)
    \envert{
      f(x) - g_j(x)
    }
    + p_{j+1}(x)
    \envert{
      f(x) - g_{j+1}(x)
    }
    +
    \eps
    +
    \eps
    \\
    &\leq
    p_j(x) \eps + p_{j+1}(x)\eps + 2\eps.
  \end{align*}
\end{proof}

\begin{proof}[Proof of \Cref{fact:relu_invert_nounity}]
  Set $c := 1/b$ and $r := \lceil b \ln(1/(\eps a)) / a\rceil$
  and $\eps_0 := \eps / (r^2 c)$.
  For $i \in \cbr{0,\ldots,r}$,
  let $h_i : [0,1]\to[0,1]$ denote a ReLU network
  $\eps_0$-approximation to the map $x\mapsto x^i$;
  by \Cref{fact:relu_exp},
  $h_i$ has $\cO(\ln(1/\eps_0)^2$ nodes and layers.
  Define $q : [0,1] \to \R$ as
  \[
    q(x) := c \sum_{i=0}^r h_i(1-cx).
  \]
  By construction, $q$ is a ReLU network with
  $\cO(r\ln(1/\eps_0)^2)$ nodes and $\cO(\ln(1/\eps_0)^2)$ layers.

  For the approximation property of $q$,
  let $x \in [a,b]$ be given, and note
  \begin{align*}
    \envert{ q(x) - \frac 1 x}
    &\leq
    \envert{ q(x) - c \sum_{i=0}^r (1-cx)^i }
    +
    \envert{ c\sum_{i=0}^r (1-cx)^i - \frac 1 x}
    \\
    &\leq
    c\sum_{i=0}^r
    \envert{ h_i(1-cx) - (1-cx)^i }
    +
    \envert{ c\sum_{i=0}^r (1-cx)^i - \frac c {1-(1-cx)} }
    \\
    &\leq
    \eps
    +
    \envert{ c\sum_{i=0}^r (1-cx)^i - c \sum_{i=0}^\infty (1-cx)^i }
    \\
    &=
    \eps
    +
    c \sum_{i=r+1}^\infty (1-cx)^i
    \\
    &=
    \eps
    +
    \frac {c(1-cx)^{r+1)}}{1-(1-cx)}
    \\
    &\leq
    \eps
    +
    \frac {\exp(-cx(r+1))}{x}
    \\
    &\leq
    \eps
    +
    \frac {\exp(-car)}{a}
    \\
    &\leq
    \eps
    +
    \eps.
  \end{align*}
\end{proof}

\begin{proof}[Proof of \Cref{fact:relu_invert}]
  Set $\eps_0 := \eps/3$.
  For $i \in \cbr{1,\ldots,k}$,
  Let $\tilde q_i$ denote the ReLU network $\eps_0$-approximation to
  $1/x$ along $[2^{-i}, 2^{-i + 1}]$;
  by \Cref{fact:relu_invert_nounity},
  $\tilde q_i$ has $\cO(k^2 \ln(1/\eps)^2)$ layers
  and $\cO(k^3 \ln(1/\eps)^3)$ nodes.
  Furthermore, set $q_i := \max\{ 2^i, \min\{ 0, \tilde q_i\}\}$,
  which has the same approximation and size properties of $\tilde q_i$.
  Applying \Cref{fact:relu_part_unity}
  with $B := 2^k$
  and reals $a_i := 2^{i-k-1}$ for $i\in\{0,\ldots, k+2\}$
  and functions $(q_i)_{i=1}^k$,
  it follows that there exists $q:\R\to\R$ which $\eps$-approximates
  $1/x$ along $[2^{-k}, 1]$
  with size $\cO( k^2 \ln(1/\eps) + k^4 \ln(1/\eps)^3)$
  and depth $\cO( k\ln(1/\eps) + k^2 \ln(1/\eps)^2 )$.
\end{proof}

Putting the pieces together gives the proof of the second part
of the main theorem.

\begin{proof}[Proof of part 1 of \Cref{fact:main}]
  Define $\epsilon_0 := \epsilon / 2^{2k + 3}$,
    and use
  \Cref{fact:relu_mul,fact:relu_polynomial,fact:relu_invert}
  to choose ReLU network approximations
  $f_p$ and $f_q$ to $p$ and $q$ at resolution $\eps_0$,
  as well as ReLU network $f$ for multiplication along $[0,1]^2$
  and $g$ to approximate $x\mapsto 1/x$ along $[2^{-k-1},1]$, again at resolution $\eps_0$.
  The desired network will compute the function $h$, defined as
  \[
    h(x) := 2^{k+1} f(f_p(x), 2^{-k-1} g(f_q(x))).
  \]
  Combining the size bounds from the preceding lemmas, $h$ itself has size bound
  \begin{gather*}
    \cO\del{ \min\cbr{ sr\ln(sr/\eps_0), sd\ln(dsr/\eps_0)^2 } }
    + \cO\del{ \ln(1/\eps_0) }
    + \cO\del{ k^4 \ln(1/\eps_0)^3 }
    \\
    =
    \cO\del{ \min\cbr{ srk\ln(sr/\eps), sdk^2 \ln(dsr/\eps)^2 }
    + k^7 \ln(1/\eps)^3 }.
  \end{gather*}
  Before verifying the approximation guarantee upon $h$,
  it is necessary to verify that the inputs to $f$ and $g$ are of the
  correct magnitude, so that \Cref{fact:relu_mul,fact:relu_invert} may be applied.
  Note firstly that $g(f_q(x)) \in [1,2^{k+1}]$,
  since $q(x) \in [2^{-k},1]$
  implies $f_q(x) \in [2^{-k} - \eps_0, 1] \subseteq [2^{-k-1},1]$.
  Thus $2^{-k-1}g(f_q(x)) \in [0,1]$, and so both arguments to $f$ within the definition of $h$
  are within $[0,1]$.
  Consequently,
  the approximation guarantees of
  \Cref{fact:relu_mul,fact:relu_polynomial,fact:relu_invert} all hold,
  whereby
  \begin{align*}
    h(x) - \frac {p(x)}{q(x)}
    &=
    2^{k+1} f\del{ f_p(x), 2^{-k-1} g(f_q(x)) } - \frac {p(x)}{q(x)}
    \\
    &\leq
    f_p(x) g(f_q(x)) - \frac {p(x)}{q(x)} + 2^{k+1} \eps_0
    \\
    &\leq
    \frac{f_p(x)}{f_q(x)}  - \frac {p(x)}{q(x)} + f_p(x)\eps_0 + 2^{k+1} \eps_0
    \\
    &\leq
    \frac{p(x)+\eps_0}{q(x) - \eps_0}  - \frac {p(x)}{q(x)}  + f_p(x)\eps_0 + 2^{k+1} \eps_0
    \\
    &\leq
    \frac{p(x)q(x) +q(x) \eps_0 - p(x)q(x) + p(x) \eps_0}{q(x)(q(x) - \eps_0)}  + f_p(x)\eps_0 + 2^{k+1} \eps_0
    \\
    &=
    \frac{\eps_0}{q(x) - \eps_0}
    + \frac{p(x) \eps_0}{q(x)(q(x) - \eps_0)}
    + f_p(x)\eps_0 + 2^{k+1} \eps_0
    \\
    &\leq
    2^{k+1} \eps_0
    + 2^{2k+1} \eps_0
    + f_p(x)\eps_0 + 2^{k+1} \eps_0
    \\
    &\leq \eps.
  \end{align*}
  The proof of the reverse inequality is analogous.
\end{proof}

\subsection{Proof of \Cref{fact:div:shallow}}

\begin{proof}[Proof of \Cref{fact:div:shallow}]
  By \citep[Lemma 2.1]{mjt_easy_relu}, a ReLU network $g$ with
  at most $m$ nodes in each of at most $l$ layers computes a
  function which is affine along intervals forming a partition
  of $\R$ of cardinality at most $N' \leq (2m)^l$.
  Further subdivide this collection of intervals at any point
  where $g$ intersects $f(x) = 1/x$;
  since $f$ is convex and $g$ is affine within each existing
  piece of the subdivision, then the number of intervals is at most
  three times as large as before.
  Together, the total number of intervals $N''$ now satisfies
  $N''\leq 3(2m)^l$.
  Finally, intersect the family of intervals with $[1/2,3/4]$,
  obtaining a final number of intervals $N \leq 3(2m)^l$.

  Let $(U_1,\ldots,U_N)$ denote this final partition of $[1/2,3/4]$,
  and let $(\delta_1,\ldots,\delta_N)$ denote the corresponding interval
  lengths.
  Let $S \subseteq \{1,\ldots,N\}$
  index the subcollection of intervals with
  length at least $1/(8N)$, meaning
  $S := \{ j \in \{1,\ldots,N\} : \delta_j \geq 1/(8N) \}$.
  Then
  \[
    \sum_{j \in S} \delta_j
    = \frac 1 4 - \sum_{j \not \in S} \delta_j
    > \frac 1 4 - \frac {N}{8N}
    = \frac 1 8.
  \]

  Consider now any interval $U_j$ with endpoints $\{a,b\}$.
  Since $1/2 \leq a < b\leq 3/4$, then $f$ satisfies
  $128/27 \leq f'' \leq 16$.
  In order to control the difference between $f$ and $g$ along
  $U_j$, consider two cases: either $f \geq g$ along this interval,
  or $f\leq g$ along this interval (these are the only two cases
  due to the subdivisions above).
  \begin{itemize}
    \item
      If $f\geq g$, then $g$ can be taken to be a tangent to $f$
      at some point along the interval $[a,b]$ (otherwise, the distance
      can always be only decreased by moving $g$ up to be a tangent).
      Consequently, $g(x) := f(c) + f'(c)(x-c)$ for some $c\in[a,b]$,
      and by convexity and since $f'' \geq 128/27$ over this interval,
      \begin{align*}
        \int_a^b |f(x)-g(x)| \dif x
        &\geq \min_{c\in[a,b]}\int_a^b \del{(f(c) + f'(c)(x-c) + f''(b)(x-c)^2/2) - (f(c) + f'(c)(x-c))}\dif x
        \\
        &= \min_{c\in[a,b]}\int_a^b f''(b)(x-c)^2/2 \dif x
        \\
        &\geq \frac {64}{27} \min_{c\in[a,b]}\del{\frac {(b-c)^3 - (a-c)^3}{3}}.
        \\
        &= \frac {64}{81} \min_{\alpha\in[0,1]}\del{(\alpha(b-a))^3 + ((1-\alpha)(b-a))^3}
        \\
        &= \frac {16(b-a)^3}{81}.
      \end{align*}

    \item
      On the other hand, if $g\geq f$,
      then $g$ passes above the secant line $h$ between $(a,f(a))$
      and $(b,f(b))$.  The area between $f$ and $g$ is at least
      the area between $f$ and $h$, and this latter area
      is bounded above by a triangle of width $(b-a)$ and
      height
      \begin{align*}
        \frac {f(a) + f(b)}{2} - f\del{(a+b)/2}
        &=
        \frac 1 2 \del{
          \frac 1 a + \frac 1 b - \frac 1 {a+b}
        }
        \\
        &=
        \frac 1 {2ab(a+b)} \del{
          b(a+b) + a(a+b) - ab
        }
        \\
        &\geq
        \frac 1 {2ab(a+b)} \del{
          b(a+b) + a(a+b) - ab
        }
        \\
        &\geq
        \frac {3/2}{4}.
      \end{align*}
      Combining this with $b-a\leq 1/4$,
      the triangle has area at least $3(b-a)/16 \geq 3(b-a)^3$.
  \end{itemize}
  Combining these two cases and summing across the intervals of
  $S$ (where $j\in S_j$ implies $\delta_j \geq 1/(8N)$),
  \begin{align*}
    \int_{[1/2,3/4]} |f(x)-g(x)| \dif x
    &\geq
    \sum_{j \in S}
    \int_{U_j} |f(x)-g(x)| \dif x
    \\
    &\geq
    \sum_{j \in S} \frac {\delta_j^3}{6}
    \\
    &\geq
    \frac 1 {6 (8N)^2} \sum_{j \in S} \delta_j
    \\
    &\geq
    \frac 1 {27648 (2m)^{2l}}.
  \end{align*}
  If $m < (27648\eps)^{-1/(2l)}/2$, then
  \[
    \int_{[1/2,3/4]} |f(x)-g(x)| \dif x
    \geq
    \frac 1 {27648 (2m)^{2l}} > \eps.
  \]
\end{proof}

\end{document}